\documentclass[9pt,shortpaper,twoside,web]{ieeecolor2}
\usepackage{generic}
\usepackage{cite}
\usepackage{amsmath,amssymb,amsfonts}
\usepackage{algorithmic}
\usepackage{graphicx}
\usepackage{textcomp}
\usepackage[T1]{fontenc}
\usepackage{multirow,multicol,graphicx}

\newtheorem{proposition}{Proposition}

\usepackage[hidelinks]{hyperref}
\def\BibTeX{{\rm B\kern-.05em{\sc i\kern-.025em b}\kern-.08em
    T\kern-.1667em\lower.7ex\hbox{E}\kern-.125emX}}
\begin{document}
\title{EEG-PRISM: Projection-based Interpretable Signal Modeling for Clinically-Grounded Explainability in EEG Foundation Models}
\title{EEG-PRISM: Physiologically-Grounded Interpretability of Predictions \\ by EEG Foundation Models}
\author{Deeksha M. Shama, \IEEEmembership{Graduate Student Member, IEEE}, Punnisa Amornsirikul, Archana Venkataraman \IEEEmembership{Senior Member, IEEE}
\thanks{This is a pre-print. Manuscript is currently under peer-review}
\thanks{This work was supported in part by NIH R01DC022565 (PI: Venkataraman) and NIH R01HD108790 (PI: Venkataraman). }
\thanks{D. M. Shama is currently with the Department of Electrical and Computer Engineering at Johns Hopkins University, MD, USA 21218 and visiting student researcher at Boston University, MA, USA 02215 (email: dshama1@jhu.edu)}
\thanks{P. Amornsirikul is currently with Boston University, MA, USA 02215 (email: punnisa@bu.edu)}
\thanks{A. Venkataraman was with Johns Hopkins University and is currently with the Department of Electrical and Computer Engineering at Boston University, MA, USA 02215 (email: archanav@bu.edu)}}

\maketitle

\begin{abstract}
\textcolor{cyan}{\textit{Objective:}}
Foundation models represent the next advancement in AI for EEG analysis; however current explainable AI techniques provide attribution scores in the time-channel input space, which is mismatched to clinical intuition about EEG. Thus, there is a critical need for a universal method that can extend the interpretability of \textit{any} foundation model to alternative and physiologically relevant domains \textit{without modifying or re-training the underlying model}.
\textcolor{cyan}{\textit{Methods:}} EEG-PRISM leverages linear transformations and established backpropagation rules to map time-channel attribution scores into alternative domains. We derive mappings to the frequency domain via an invertible DFT and to the source domain via an approximately invertible EEG generative model. We evaluate EEG-PRISM in simulated and real data, assessing recovery of ground-truth phenomena across domains with five foundation models and four AI explainers.
\textcolor{cyan}{\textit{Results:}} In simulation, EEG-PRISM achieves near-perfect spectral recovery and 69.2\% spatial accuracy. In epilepsy, EEG-PRISM correctly determines that delta–theta activity is most salient and correctly localizes the seizure onset region with 50\% accuracy. In autism, EEG-PRISM localizes the predictive delta–alpha biomarkers to frontal and temporal regions, consistent with prior work.
\textcolor{cyan}{\textit{Conclusion:}} EEG-PRISM is a theoretically-grounded post-hoc attribution method with accurate mapping into the spectral and spatial domains. It supports window-level analysis of transient events (e.g., seizures) and group-level identification of clinically relevant biomarkers (e.g., autism), thus advancing interpretable EEG foundation models.
\textcolor{cyan}{\textit{Significance:}} This work enables physiologically-grounded interpretation of EEG foundation models and supports clinically relevant insights such as event localization and biomarker identification.
\end{abstract}


\section{Introduction}

Scalp electroencephalography (EEG) provides a non-invasive window into brain activity by capturing the rich temporal dynamics of underlying neural processes. In addition, variations in the structured spatial and spectral activity of EEG signals are closely linked to neurological disorders and pathology. For example, in the widely-studied condition of epilepsy, synchronous spiking activity and evolving spectral patterns across frequency bands in the EEG signals provide important insights into seizure dynamics~\cite{tharp1975spectral, goenka2018comparative}; they are also used to distinguish seizure subtypes and guide treatment options \cite{zaveri2001distinguishing, ouyang2018quantitative}. In the research of psychiatric conditions using EEG \cite{hughes1999conventional}, studies have linked alterations in spectral activity, primarily in the lower frequency bands, to disorders such as ADHD and OCD~\cite{newson2019eeg}. While findings in autism remain mixed, there is some evidence suggesting differences in higher-frequency (e.g., gamma) activity \cite{neo2023resting}.  

In the spatial domain, EEG signals are modeled as the result of electrical activity generated by distributed neural sources which propagate through head tissues to the scalp according to a well-established forward model \cite{hallez2007review}. Particularly with high density recordings, there is increasing interest in leveraging these spatial relationships to ``invert" the forward model and map the EEG activity to cortical regions in the brain. For example, scalp EEG can  offer approximate localization of seizure onset regions and serve as a noninvasive precursor to invasive evaluation of epileptogenic zones~\cite{ray2007localizing}. Source-level interpretations are also increasingly leveraged for psychiatric conditions~\cite{coburn2006value}.

Despite its clinical value, EEG is complex, high-dimensional, and often contaminated by artifacts. These factors make manual interpretation challenging, particularly for continuous and multichannel recordings. The rise of artificial intelligence has led to computational methods for EEG analysis, with increasing emphasis on learning representations directly from EEG data that generalize across datasets and tasks \cite{sharma2024emerging}. While these models can achieve high prediction accuracies, their clinical utility is limited by their lack of explainability, as shown in Fig.~\ref{fig:problem}. In particular, these models operate on EEG signals as input, so post-hoc explainers, such as GradCam~\cite{zhao2025interpretable} and SHAP~\cite{sylvester2024shap}, provide temporal and channel-level attribution scores. This information does not align with how EEG is understood in practice, where diagnostic insights are more naturally expressed in the spectral and spatial domains. Thus, there is a growing need for trustworthy approaches whose interpretations are not only accurate, but also grounded in physiologically meaningful representations that can integrate seamlessly into clinical workflows. 

\begin{figure}[t]
    \centering
    \includegraphics[width=0.99\linewidth]{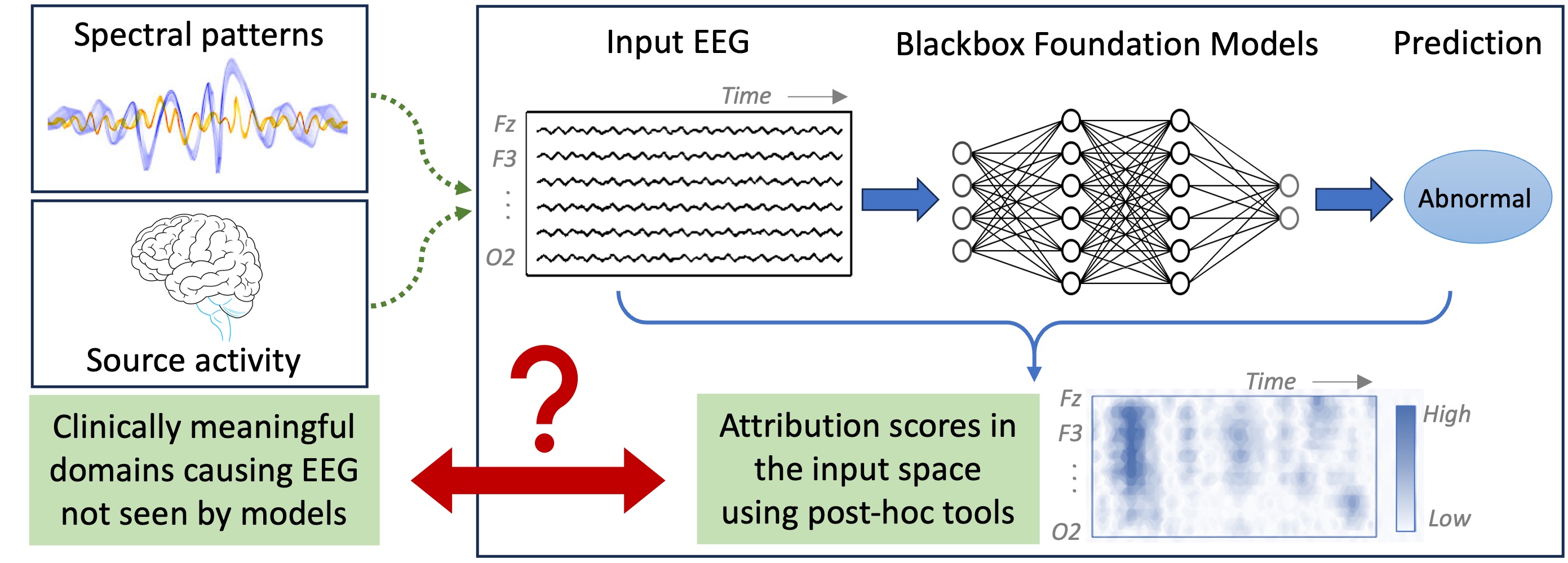}
    \caption{EEG foundation models and their explainability using post-hoc tools operate in time-channel space, while clinically meaningful patterns lie in spectral and source domains that govern EEG generation.}
    \label{fig:problem}
\end{figure}

\subsection{EEG Foundation Models as Black-Box Computational Tools}

EEG foundation models are increasingly replacing traditional feature-engineering pipelines due to their improved performance and ease of deployment. Traditional approaches are often pipelined to include preprocessing, handcrafted feature extraction, and a prediction module (e.g., logistic regression or support vector machines). The handcrafted features (e.g., Fourier transforms, wavelet decompositions, and spectral power) are motivated by neuroscientific insights, but they require extensive experimentation to identify representations that best capture the phenomenon of interest \cite{singh2023trends}. In contrast, deep neural networks learn hierarchical representations directly from the time–channel EEG inputs and reduce the need for manual feature design \cite{hornik1991approximation}. With sufficient training data, these networks can achieve greater predictive performance than traditional approaches, and they often demonstrate improved generalization across datasets \cite{kuruppu2026eeg}.

Foundation models represent the next advancement in AI for EEG analysis. They build on early deep neural networks by (1) pretraining the model on large-scale EEG datasets that comprise thousands of hours of recordings; and (2) using self-supervised learning to provide flexibility across downstream prediction tasks. This setup is in contrast to earlier deep neural networks, including convolutional architectures such as EEGNet \cite{lawhern2018eegnet} and hybrid convolution–transformer models such as EEGConformer \cite{song2022eeg}, which require task-specific training from scratch.  Current EEG foundation models adopt a range of deep neural network architectures for large-scale self-supervised representation learning.  All models learn the self-supervised representations directly from the EEG signals (i.e., time-channel inputs). Early models, such as BENDR~\cite{kostas2021bendr}, adapt a contrastive prediction framework from speech analysis to EEG, while more recent models like BIOT~\cite{yang2023biot} employ masked training to learn temporal dependencies from a large corpora. Other works focus on improving input representations, including Labram~\cite{jiang2024large}, which introduces discrete neural tokenization, and FoME~\cite{shi2024fome}, which incorporates time–frequency fusion to better capture spectral structure. Models such as Cbramod~\cite{wang2025cbramod} explicitly separate spatial and temporal attention mechanisms, whereas LUNA~\cite{doner2026luna} addresses data heterogeneity by learning topology-agnostic latent representations. 

While powerful, foundation models operate as black-box units and learn representations directly from the time-channel input data. This strategy enables rapid adaptation and strong performance across datasets, but their decision mechanisms remain difficult to interpret. Moreover, current explainable AI techniques attempt to quantify input-output relationships and cannot provide explanations in the clinically meaningful spectral and spatial domains. On the other hand, incorporating domain-relevant information into the foundation model is computationally expensive, requiring architectural changes and retraining. Taken together, there is a need for an explainability tool that directly extract physiologically-grounded explanations with minimal overhead as a crucial stepping stone towards clinical translation.

\subsection{Explainability in Deep Learning Pipelines}
Current explainable AI methods do not aim for full transparency of black-box models, but rather to address the question of \textit{why a particular prediction is made}. Given the difficulty of modifying black-box deep neural networks, many widely used methods are post hoc, operating after model training and inference. These methods are typically model-agnostic and assign ``attribution scores" to quantify the contribution of each input feature to the model prediction. Notably for EEG foundation models, the attribution scores are provided on the time-channel inputs, which may be difficult to interpret in the context of a neurological disorder \cite{presacan2026comprehensive, maurer2024explainable}. 

At a high level, post-hoc explainers use specific rules to backpropagate information through the neural network in order to compute the attribution scores. Popular methods include:
 \begin{itemize}
     \item \textbf{Layer-wise Relevance Propagation (LRP)} redistributes the model prediction backward through the network according to conservation principles involving different propagation rules, such as the $\epsilon$-rule and $\gamma$-rule for neural networks, as well as customized rules for transformers \cite{montavon2019layer, achtibat2024attnlrp, bexten2026clever, nam2023effects, nouri2024detection}.
    \item \textbf{Integrated Gradients} constructs attribution scores by integrating the gradients of the model output with respect to the input along a path from a baseline to the observed sample. This strategy ensures that the attributions are highly sensitive and invariant to implementation across equivalent models~\cite{sundararajan2017axiomatic, maiti2025neuroexplain, tachikawa2018compensated}. 
    \item \textbf{DeepLIFT} computes the attribution scores by comparing neuron activations to a reference baseline and propagating these differences through the network. This approach allows for non-zero attributions even in regions where gradients are saturated, providing more robust attribution scores \cite{shrikumar2017learning, park2023spatio, sujatha2023empirical}
    \item \textbf{SHAP} assigns input-level scores that correspond to theoretically-grounded Shapley values. DeepSHAP extends the original framework to deep networks by combining SHAP’s attribution principle with DeepLIFT-style propagation. This combination enables the scalable approximation of Shapley-consistent attributions in complex models \cite{lundberg2017unified, shawly2025eeg, mouazen2025transparent, almadhor2025interpretable, sylvester2024shap}.
 \end{itemize}

Given that EEG foundation models operate directly on time–channel inputs, attribution scores obtained using any of these post-hoc methods are expressed in the time–channel domain. However, such representations are not always interpretable or actionable by a clinician. For example, temporal attributions at the millisecond scale may not align with clinically relevant patterns, which often span longer durations. Similarly, channel-wise attributions can be misleading due to strong correlations between electrodes. As previously discussed, clinical phenomena in EEG are more often associated with spectral (frequency-domain) and spatial (source-domain) patterns.

\subsection{Our Contributions}

In this paper, we introduce EEG-PRISM for \textbf{P}hysiologically-g\textbf{R}ounded \textbf{I}nterpretability of Prediction\textbf{S} by EEG foundation \textbf{M}odels. Our work includes several novel contributions for EEG analysis:
\begin{itemize}
    \item A \textbf{mathematically grounded framework} to map EEG foundation model attributions into clinically meaningful frequency and source domains using DFT and inverse modeling. Our theoretical contributions include a derived error bound and guaranteed compatibility across multiple explainers via backpropagation.
    \item A \textbf{model-agnostic and post-hoc approach}  that does not require any modification to the existing foundation model; EEG-PRISM also has minimal computational overhead.
    \item \textbf{Extensive simulation studies} with known ground truth in the spectral and source domains demonstrate the attribution fidelity of EEG-PRISM across five different EEG foundation models. 
    \item \textbf{Real-world biomedical validation} in seizure localization and autism biomarker discovery. EEG-PRISM enables both subject-specific event detection and group-level neural characterization.
\end{itemize}
Taken together, we present a valuable  tool that aligns the trajectory of AI-based EEG analysis with clinical and research practice.

\begin{figure*}[t]
    \centering
    \includegraphics[width=0.95\linewidth]{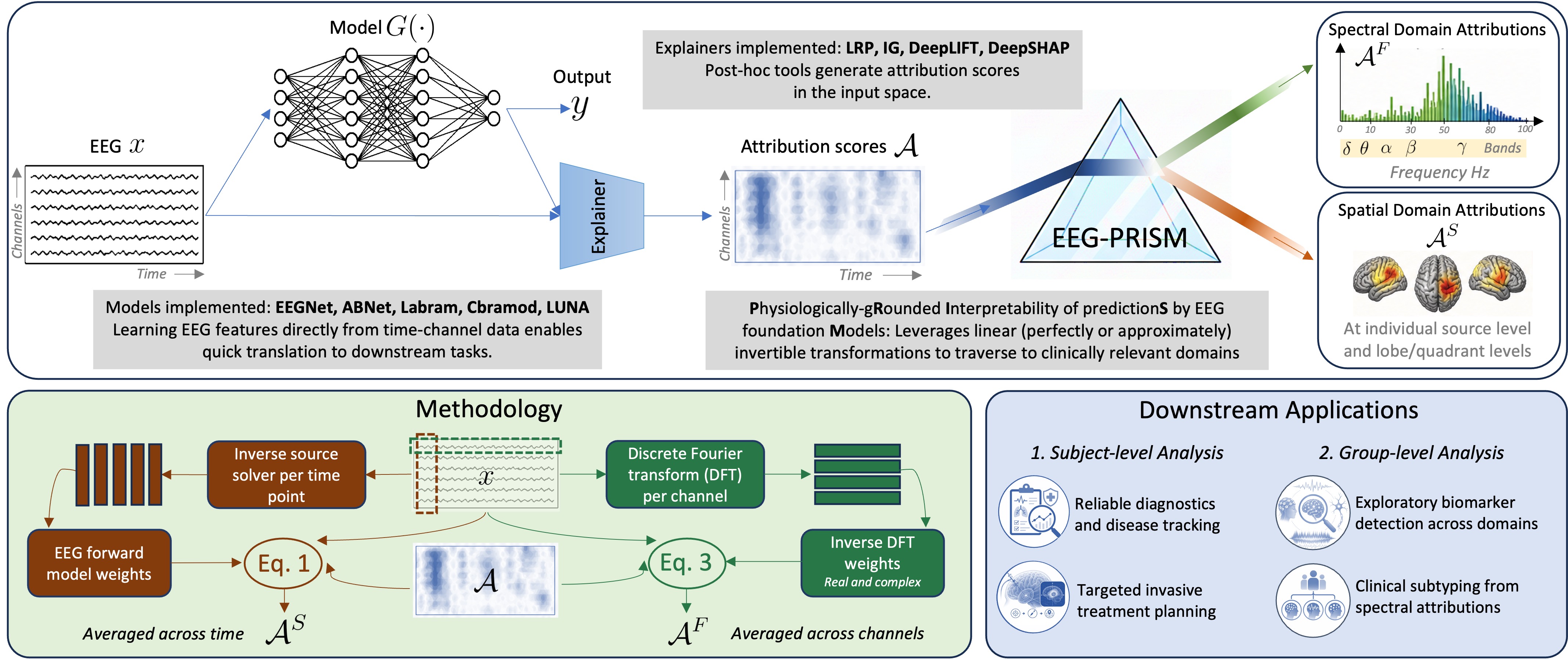}
    \caption{Overall pipeline of EEG-PRISM for explainability in clinically relevant domains beyond the input time-channel space. \textbf{Top:} End-to-end experimental setup, in which EEG-PRISM is applied to foundation models post-hoc to obtain attributions in the spectral and spatial domains. \textbf{Bottom left:} Transforming attributions to  new domains within EEG-PRISM. \textbf{Bottom right:} Potential downstream applications.}
    \label{fig:prism}
\end{figure*}

\section{Methods}

EEG-PRISM leverages the principle of \textit{attribution propagation} to map explanations  derived from foundation models that operate on the time–channel EEG input space onto clinically meaningful subspaces. Mathematically, let $X \in \mathbb{R}^{C \times T}$ denote the input EEG, where $C$ is the number of channels and $T$ the number of time points. The AI foundation model $G(\cdot)$ uses $X$ for a prediction task (e.g., case/control classification), leading to post-hoc attributions $\mathcal{A} \in \mathbb{R}^{C \times T}$ in the input space. Our goal is to transform $\mathcal{A}$ into (1)~spectral attributions $\mathcal{A}^F$ that capture predictive frequency components, and (2)~spatial attributions $\mathcal{A}^S$ that capture importance across source locations. Crucially, the transformation should not require modification or retraining of $G(\cdot)$. Noting that the spectral domain can be reached from the input space via an invertible discrete Fourier transform, while the spatial domain arises by inverting the EEG forward model, the full pipeline for EEG-PRISM is shown in Fig.~\ref{fig:prism}.

\subsection{Mapping Attribution Scores to Clinically Relevant Domains}

Consider a generic linear transform between the input and desired subspace. Formally, let $\mathbf{x} \in \mathbb{R}^{N}$ denote a vector in either the row space (N = T) or column space (N = C) of the EEG input $X \in \mathbb{R}^{C \times T}$. Let $\mathbf{z} \in \mathbb{R}^{M}$ with $M \geq N$ represent the corresponding vector in the target subspace, such that $\mathbf{x} = W \mathbf{z}$, where $W \in \mathbb{R}^{N \times M}$ is a well-defined linear transformation.
EEG-PRISM leverages the following proposition to directly map the attribution scores:

\begin{proposition}
    The attribution scores in the $z$-space, denoted by $\mathcal{A}^z$, can be expressed as signal times a weighted linear combination of the original attribution scores. Let $\mathcal{A}^x$ denote the row or column of $\mathcal{A}$ corresponding to the input vector $\mathbf{x}$. Then we have:
\begin{equation} \label{eq:map}
    \mathcal{A}^z_i = \mathbf{z}_i \sum_{j=1}^{N} W_{ji} \frac{\mathcal{A}^x_j}{\mathbf{x}_j},
\end{equation}
where the weights~$W_{ji}$ correspond to the linear transformation. 
\end{proposition}

\begin{proof}
    Treating the subspace transformation as a linear change of variables, we define the composite function $\tilde{G}(\mathbf{z}) := G(W \mathbf{z}) = G(\mathbf{x}) = y$. Since $W$ is fixed and linear, all attribution rules of post-hoc explainers apply through the standard composition. We outline the proof for three representative classes of explainers. 

\smallskip
    \paragraph{Layer-wise Relevance Propagation (LRP)} The proof of Proposition~1 for LRP is detailed in~\cite{vielhaben2024explainable} and restated here for completeness. Using standard relevance propagation for linear layers and the conservation principle, we obtain the relationship:
    \begin{equation}\label{eq:lrp}
        \mathcal{A}^z_i =  \sum_{j=1}^{N} \frac{W_{ji} \mathbf{z}_i}{\sum_{k} W_{jk} \mathbf{z}_k} \mathcal{A}^x_j.
    \end{equation}
    Note that Eq.~\eqref{eq:lrp} is equivalent to Eq.~\eqref{eq:map} as $\sum_{k} W_{jk} \mathbf{z}_k = \mathbf{x}_j$.

\smallskip
    \paragraph{Gradient-Based Explainers} This proof hinges on the chain rule. Let us consider the simplest method of Gradient times Input (G$\times$I). The chain rule over the composite function $\tilde{G}(\mathbf{z})$ yields
    \begin{equation} \label{eq:chainrule}
        \frac{\partial \tilde{G}}{\partial \mathbf{z}_i}
        = \sum_{j=1}^{N} \frac{\partial G}{\partial \mathbf{x}_j} \cdot \frac{\partial \mathbf{x}_j}{\partial \mathbf{z}_i}
        = \sum_{j=1}^{N} \frac{\partial G}{\partial \mathbf{x}_j} \cdot W_{ji} .
    \end{equation}

    Thus, the attribution propagates linearly through $W$ as follows:
    \begin{equation} \label{eq:gxi}
        A^z_i = \mathbf{z}_i \frac{\partial \tilde{G}(\mathbf{z})} {\partial \mathbf{z}_i}  =  \mathbf{z}_i  \sum_{j} \frac{\partial {G}(\mathbf{x})} {\partial \mathbf{x}_j} W_{ji} =  \mathbf{z}_i  \sum_{j} W_{ji} \frac{\mathcal{A}^x_j}{\mathbf{x}_j} 
    \end{equation}     

    Consider Integrated Gradients (IG), which reduces to a path-integral accumulation of input gradients that satisfies crucial axioms of completeness and implementation invariance~\cite{sundararajan2017axiomatic}. By combining the chain rule in Eq.~\eqref{eq:chainrule} and the linearity property of integration given $x = Wz$, we obtain the following attributions for IG:
    \begin{align} \nonumber
        A^z_i &= \mathbf{z}_i \int_\alpha \frac{\partial \tilde{G}(\alpha \mathbf{z})} {\partial \mathbf{z}_i} d\alpha = \mathbf{z}_i \int_\alpha \textstyle \sum_{j} \frac{\partial {G}(\alpha \mathbf{x})} {\partial \mathbf{x}_j} W_{ji} \,d\alpha \\[2ex]
        &  =  \mathbf{z}_i \sum_{j=1}^N W_{ji} \frac{\mathcal{A}^x_j}{\mathbf{x}_j}
    \end{align}
    Note that the final step is a direct result of the definition of Integrated Gradients in the $x$-space, $\mathcal{A}^x_j =\mathbf{x}_j \int_\alpha \frac{\partial {G}(\alpha \mathbf{x})} {\partial \mathbf{x}_j} d\alpha $.

\smallskip
    \paragraph{DeepLIFT and DeepSHAP} The proof for this class of explainers uses the backpropagation and chain rules of discrete differentials~$m$. This formulation leads to linear redistribution under fixed affine transformations. Specifically, the chain rule states
    \begin{equation}
        m_{\Delta_{\mathbf{z}_i} \Delta_{y}}  = \sum_j m_{\Delta_{\mathbf{z}_i} \Delta_{\mathbf{x}_j}} m_{\Delta_{\mathbf{x}_j} \Delta_{y}},
    \end{equation}
    where $\Delta_{(\cdot)}$ is the difference from a fixed reference. For the linear relation $ \mathbf{x} = W\mathbf{z}$, DeepLIFT and DeepSHAP assign $m_{\Delta_{\mathbf{z}_i} \Delta_{\mathbf{x}_j}} = W_{ji}$. Further, using the equality $m_{\Delta_{\mathbf{x}_j} \Delta_{y}} = \frac{\mathcal{A}^x_j}{ \Delta_{\mathbf{x}_j}}$ we have    
    \begin{equation} \label{eq:deeplift}
        \mathcal{A}^z_i = \Delta_{\mathbf{z}_i} \cdot m_{\Delta_{\mathbf{z}_i} \Delta_{y}}  =  \Delta_{\mathbf{z}_i} \cdot \sum_j W_{ji}\frac{\mathcal{A}^\mathbf{x}_j}{ \Delta_{\mathbf{x}_j}}.
    \end{equation}
    By setting the reference to zero, i.e., $\Delta_{\mathbf{z}_i} = \mathbf{z}_i$ and  $\Delta_{\mathbf{x}_i} = \mathbf{x}_i$, Eq.~\eqref{eq:deeplift} becomes equivalent to Eq.~\eqref{eq:map} in the proposition. 
    
    For DeepSHAP, which builds on DeepLIFT, the proof follows the same structure with empirical average over multiple references.        
     \hfill $\square$
\end{proof}

\begin{figure*}[!ht]
    \centering
    \includegraphics[width=0.8\linewidth]{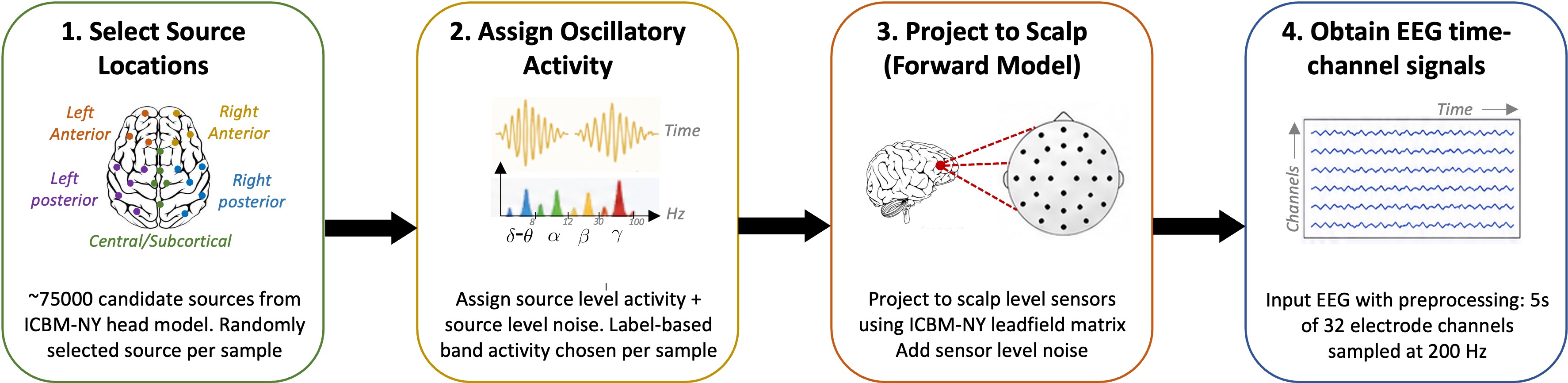}
    \caption{Simulated EEG data generation pipeline using SEREEGA. (1) Initialize the head model and select a source location, (2) Assign class-associated frequency signal, (3) Generate EEG from source activity using the leadfield matrix, and (4) Preprocess the EEG using standard methods.}
    \label{fig:sim}
\end{figure*}
 
\medskip \noindent
Taken together, Eq.~\eqref{eq:map} is theoretically grounded and preserves the essential properties of the base explainers (e.g., faithfulness, conservation, and stability) without modifying $G(\cdot)$. EEG-PRISM leverages this result to map time-channel input attribution scores to the spectral and spatial domains as described below.

\medskip
\subsubsection{Spectral Domain}
For each EEG channel, the spectral representation $\mathbf{z}^c_{1:T}$ of the EEG signal $\mathbf{x}^c_{1:T}$ is obtained via the Discrete Fourier transform (DFT), which is a predefined linear and invertible operation. The caveat is that a DFT produces complex-valued coefficients, which would, in turn, lead to complex-valued attribution scores due to the multiplication in Eq.~\eqref{eq:map}. As seen in~\cite{vielhaben2024explainable}
we decompose the DTF coefficients~$\mathbf{z}^c_{1:T}$ into its real and imaginary components. The inverse DFT to reconstruct the EEG data $\mathbf{x}^c_{1:T}$ can be written
\begin{equation} \label{eq:idft}
    \mathbf{x}^c_j = \sum_{i=1}^T \text{Re}(\mathbf{z}^c_i) \text{cos}\left(\frac{2\pi ji}{T}\right) - \text{Im}(\mathbf{z}^c_i) \text{sin}\left(\frac{2\pi ji}{T}\right).
\end{equation}
where $\text{Re}(\cdot)$ and $\text{Im}(\cdot)$ denote the real and imaginary components, respectively. From here, we create a new vector representation by concatenating the real and imaginary parts: $[\text{ Re}(z^c)  \text{ } | \text{ Im}(z^c) \text{ } ] $. 
The linear weight $W_{ij}$ to map the attributions from the time domain to the spectral domain is obtained from Eq.~\eqref{eq:idft} as $\cos\!\left({2\pi j i/T}\right)$ for the real part and $-\sin\!\left({2\pi j i/T}\right)$ for the imaginary part of the vector. 

Substituting these expressions into Eq.~\eqref{eq:map} and adding the attribution scores of the real and imaginary components gives us the spectral attributions per channel as follows:
\begin{equation}
    \mathcal{A}^{F}_i = \frac{1}{C} \sum_{c=1}^C\sum_{j=1}^T \left[  \text{Re}(\mathbf{z}^c_i) \text{cos}\left(\frac{2\pi ji}{T}\right) - \text{Im}(\mathbf{z}^c_i) \text{sin}\left(\frac{2\pi ji}{T}\right) \right] \frac{\mathcal{A}^{c}_j}{x^c_j} 
\end{equation}

Finally, averaging across all electrode channels provides the overall spectral-domain attribution $\mathcal{A}^{F}$ for the EEG data. 

\medskip
\subsubsection{Spatial Domain} 
The relationship between the multichannel EEG data at time point~$t$, i.e., $\mathbf{x}^{1:C}_t$ and the corresponding spatial source domain~$\mathbf{z}^{S}_t$ is modeled via the EEG forward model $\mathbf{x}^{1:C}_t = W \mathbf{z}^{S}_t$. Here, $\mathbf{z}^{S}_t \in \mathbb{R}^{M}$ denotes the instantaneous source-level (e.g., cortical) activity across $M$ sources, and $W \in \mathbb{R}^{C \times M}$ is the forward matrix derived from biophysically informed head conductivity models that assume a fixed normal dipole orientation. Since the EEG forward problem provides a well-defined linear mapping, the transformation weights $W_{ji}$ in Eq.~\eqref{eq:map} can be directly obtained from the in-built matrix. A straightforward substitution of these weights along with $z^{S}_t$ yields instantaneous attribution scores, which are then averaged across time to obtain the source domain attribution, $\mathcal{A}^{S}$.

The primary consideration in applying Eq.~\eqref{eq:map} is that of estimated the source-level activity lies in obtaining the source activity $\mathbf{z}^{S}_t$ required for the computation. Estimating $\mathbf{z}^{S}_t$ corresponds to the EEG inverse problem, for which a range of well-established solvers exist. Since the number of cortical sources is often greater than the number of EEG channels, i.e., $M \gg C$, the problem is inherently ill-posed and requires appropriate regularization to invert.

Importantly, EEG-PRISM does not amplify or introduce errors into the inverse mapping. Rather, EEG-PRISM operates post-hoc and propagates attribution scores through a linear mapping. As a result, any error in the estimated source $\mathbf{z}^{S}$ is transferred linearly, with no additional distortion beyond a constant scaling factor determined by the transformation. Formally, if the true source $\mathbf{z}^{\text{true}}$ differs from the estimated $\mathbf{z}^{S}$, then the induced error in the mapped attribution remains proportional to this discrepancy as follows:
\begin{equation} \label{eq:approx}
    \left\Vert \mathcal{A}^{true} - \mathcal{A}^{S} \right\Vert \leq \left\Vert z^{true} - z^{S}\right\Vert \cdot \underbrace{\left\Vert\sum_{j=1}^{N} w_{ji} \frac{\mathcal{A}^j}{x^j}\right\Vert}_{\text{constant}}
\end{equation}
Thus, EEG-PRISM can seamlessly project attributions to the spatial domain, while preserving linear fidelity and not introducing any distortion beyond what is inherent in the inverse solver.

\subsection{Datasets}
Our experiments include three datasets: a simulated dataset for quantitative evaluation of EEG-PRISM against a known ground truth, and two real-world datasets used to demonstrate clinical potential. 

\medskip \noindent
\subsubsection{Simulated Dataset} We generate EEG using SEREEGA \cite{krol2018sereega} with its built-in source-to-scalp forward modeling framework, as shown in Fig.~\ref{fig:sim}. A 32-channel Biosemi electrode montage is employed, along with the ICBM New York head model for EEG forward modeling with 74,382 sources \cite{huang2016new}. We construct ``classes" with labels that are explicitly tied to the underlying source-level signal characteristics. This strategy enables a controlled evaluation of EEG-PRISM across different settings. The dataset consists of four classes, each defined by a distinct frequency band in the source signal domain: delta/theta (1–8 Hz), alpha (8–12 Hz), beta (13–30 Hz), and gamma (30–100 Hz). For each sample, a source is randomly selected from the brain, and its activity is assigned oscillatory components within the frequency range associated with its class. The resulting signals are projected to the scalp to form the time–channel EEG inputs. A total of 100 samples per class (400 total) are generated. Thus, classification depends on learning frequency-specific representations internally within the foundation model from the time-channel input. 

\medskip \noindent
\subsubsection{TUSZ Dataset} We selected 124 subjects (58~M, 66~F) with focal seizure onsets from the publicly accessible Temple University Hospital Seizure (TUSZ) corpus \cite{shah2018temple}. Patients ranged in age from 19 to 91 years (55.2 $\pm$ 16.6 years). The TUSZ corpus contains several 19-channel average-referenced EEG recordings per epilepsy patient, along with expert annotations of seizure intervals. We randomly crop 30 seconds of non seizure data and up to 30 seconds of seizure data from every EEG recording of every patient; the cropped data is segmented into non-overlapping 2s windows, resulting in 39,123 non seizure and 18,603 seizure windows in our final TUSZ dataset.

The EEG foundation models are adapted for binary classification of seizure versus non-seizure windows. We use the unstructured clinical notes to extract the seizure onset channels and relevant signal frequency bands for each patient. The patient-level seizure onsets are organized as: 31 left frontal, 14 right frontal, 26 left posterior, 45 right posterior, and 8 central. While the EEG foundation models do not incorporate the onset location, this spatial information is clinically relevant. Therefore, we quantify the alignment between the spatial attributions of EEG-PRISM and the clinician-determined onsets. 

\medskip \noindent
\subsubsection{ACE Dataset} We use resting-state EEG data from the Autism Center for Excellence (ACE) project (NDAR study \#2021)~\cite{neuhaus2021resting}; the dataset includes participants with autism spectrum disorder (ASD) and typically developing controls (TDC). Of the 339 enrolled participants, N=177 (87 ASD, 90 TDC) met our inclusion criteria, which required artifact-free EEG, valid diagnostic and behavioral assessments, and no familial relatedness. Data were acquired using a 128-channel EGI Net Amps 300 system with HydroCel nets. Recordings consisted of alternating task blocks, from which eyes-open resting-state EEG segments were extracted for analysis.

The EEG foundations models are adapted for subject-level ASD versus TDC classification using all the available time-channel EEG per participant. We evaluate the spectral and spatial attribution maps provided by EEG-PRISM against findings in the ASD literature.

\medskip \noindent
\subsubsection*{Preprocessing}
We apply a uniform preprocessing pipeline across all datasets. The EEG signals are downsampled to 200 Hz and segmented into 5-second windows for the simulated dataset, and 2-second windows for real-word datasets. The 5-second windows in the simulated dataset is selected to match the default configurations of the EEG foundation models. The smaller window size in real-word datasets were chosen to improve the seizure detection resolution and to match the experimental protocol of ACE dataset. We applied a band-pass filter between 0.1-100 Hz for the simulated and ACE datasets and between 0.1-30 Hz for the TUSZ dataset. We also apply a 60 Hz notch filter to remove power line noise. Finally, we clip the signals at two standard deviations from mean to remove high intensity artifacts. The time-channel signals are normalized to have zero mean and unit variance and input to the EEG foundation models.

\subsection{Experimental Setup}
The experimental setup comprises of three main steps: (i)~Training the EEG foundation models (ii) Computing input attribution scores and (iii) Mapping the scores to spectral and spatial domains. Our code will be available in our GitHub repository\footnote{https://github.com/deeksha-ms/EEG-PRISM} upon acceptance.

\medskip \noindent
\textbf{Foundation Models:} We apply EEG-PRISM to models encompassing five deep network architectures. The two task-specific models, EEGNet \cite{lawhern2018eegnet} and AttentionBaseNet (ABNet)~\cite{wimpff2024eeg}, are trained from scratch. Three larger generic models, LaBraM~\cite{jiang2024large}, CBramod~\cite{wang2025cbramod}, and LUNA~\cite{doner2026luna}, are initialized from released weights and subsequently fine-tuned on each dataset. All models are implemented using the Braindecode package~\cite{braindecode} with default configurations as provided in their original releases, ensuring reproducibility across experiments.

\medskip \noindent
\textbf{Training Setup:} For each dataset, we augment the backbone foundation model with a task-specific classification head and train the entire network in an end-to-end fashion using a cross-entropy loss. 

We use the Adam optimizer~\cite{kingma2014adam} for training with with cosine learning rate scheduling. Training is capped at a maximum of 50 epochs, with early stopping based on the validation set performance to prevent overfitting. We follow a subject-independent 5-fold nested cross-validation scheme. Within each outer fold, an inner train–validation split is used to select the optimal learning rate within the [$10^{-3}-10^{-6}$] range. The selected model is then retrained using both the training and validation sets and is evaluated on the held-out test fold. Performance is reported as the mean and standard deviation of evaluation metrics across the five outer folds. All experiments are implemented in PyTorch (v2.11) and executed on a single A100 GPU.

\medskip \noindent
\textbf{Attribution Scores and Mapping with EEG-PRISM:} We select the best-performing model in every test-fold and compute the input-level attributions scores using LRP, IG, DeepLIFT, and DeepSHAP in the time–channel domain. LRP is implemented using Zennit package while the others are implemented using the Captum Python package~\cite{kokhlikyan2020captum}. IG and DeepLIFT uses a zero reference, while DeepSHAP uses random samples from training set as reference. 

We use EEG-PRISM to transform each set of attribution scores into both spectral and source-space attributions. For spectral mapping, we implement the DFT and its inverse from scratch in PyTorch in order to extract its coefficients in Eq.~\eqref{eq:map}. The DFT is applied independently to each channel using a sampling frequency of $f_s$=200 Hz. 
For spatial mapping, we use MNE-Python \cite{gramfort2013meg} to compute the forward model coefficients while assuming the standard \textit{fsaverage} template brain and a three-layer boundary element method (BEM) head model. The inverse solution is computed using sLORETA with the default regularization settings in MNE \cite{pascual2002standardized, grech2008review}. 
\bgroup
\renewcommand{\arraystretch}{1.5}
\begin{table*}[t]
    \centering
    \caption{Accuracy of EEG-PRISM derived attributions vs ground-truth features in simulated dataset. \textbf{Left:} Comparison between the top spectral band and ground truth band. \textbf{Right:} Comparison between the center of mass location and ground truth source.}
    \begin{tabular}{|c|c|c|c|c||c|c|c|c|}
    \hline
    \multirow{2}{*}{Foundation Model}       & \multicolumn{4}{c||}{Spectral Domain}    &          \multicolumn{4}{c|}{Spatial/Source Domain}\\ \cline{2-9}
    & LRP$\dagger$  & IG & DeepLIFT & DeepSHAP & LRP$\dagger$  & IG & DeepLIFT & DeepSHAP   \\ \hline
    EEGNet     & 1.000$\pm$0.000    &  0.942$\pm$0.037  &  0.940$\pm$0.039  &1.000$\pm$0.000& 0.613$\pm$0.043& 0.655$\pm$0.049    &  0.660$\pm$0.052 &   0.600$\pm$0.080     \\ \hline
    ABNet      &  1.000$\pm$0.000    & 0.988$\pm$0.008  &   0.978$\pm$0.015 & 1.000$\pm$0.000 &  0.655$\pm$0.079 & 0.658$\pm$0.088  &  0.655$\pm$0.089  & 0.595$\pm$0.088    \\ \hline
    Labram$^{\dagger}$ &  --   &  0.960$\pm$0.022   & 0.972$\pm$0.022 & 1.000$\pm$0.000 & --  & 0.692$\pm$0.054  & 0.680$\pm$0.063 & 0.638$\pm$0.072 \\ \hline
    Cbramod$^{\dagger}$ &  --  &  0.993$\pm$0.010 & 0.995$\pm$0.006  & 1.000$\pm$0.000 & -- &  0.670$\pm$0.057 & 0.638$\pm$0.079 & 0.628$\pm$0.090\\ \hline
    LUNA$^{\dagger}$ &  --   &  0.945$\pm$0.020  & 0.980$\pm$0.013 & 0.998$\pm$0.005  & -- & 0.648$\pm$0.093  &    0.645$\pm$0.083    &  0.625$\pm$0.089 \\ \hline
    \multicolumn{9}{l}{\footnotesize $\dagger$~We are unable to find a Python package that provides an LRP implementation that is compatible with transformer-based foundation models}
 
    \end{tabular}
    \label{tab:sim_results}
\end{table*}
\egroup
\begin{figure*}[t]
    \centering
    \includegraphics[width=0.99\linewidth]{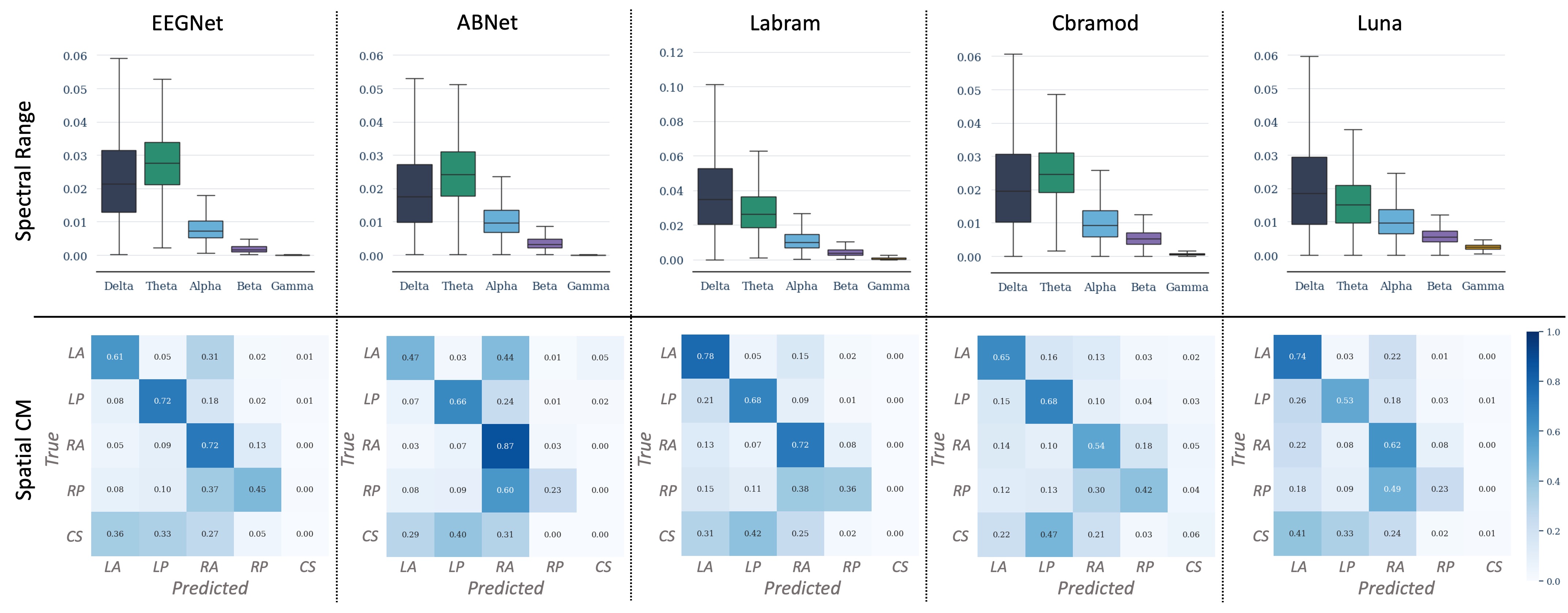}
    \caption{EEG-PRISM results on TUSZ dataset. \textbf{Top:} Range of spectral attributions aggregated into five  bands across all EEG windows in five  models. \textbf{Bottom:} Confusion Matrices (CM) comparing region with maximum aggregated spatial attributions of EEG-PRISM vs ground truth seizure onset zone in clinical notes for five models. \textit{LA: Left Anterior; RA: Right Anterior; LP: Left Posterior; RP: Right Posterior; CS: Central/Subcortical} }
    \label{fig:tuh}
\end{figure*}

\section{Results}

We evaluate the \textit{post hoc} spectral and spatial attributions provided by EEG-PRISM across several configurations. Specifically, we train each foundation model on all three datasets for the specific tasks described in Section~II.B. The models are evaluated on a held-out test set to verify proper learning. For completeness, we report the cross-validated classification performances in Appendix~\ref{section:model_acc}; however, we note that EEG-PRISM is agnostic to classification and can be applied to \textit{any} trained model. We use standard explainers to derive the input time-channel attribution scores on the test data and apply EEG-PRISM to derive the spectral and spatial attributions. We use simulated data to quantify the accuracy of the transformed attributions, and subsequently demonstrate the value-add of EEG-PRISM for recording-level signal characterization in the TUSZ dataset and group-level exploratory biomarkers in the ACE dataset.

\subsection{Simulated Dataset: Comparison with Ground Truth}

The simulated dataset contains known ground-truth spectral and spatial characteristics associated with each class label, which enables us to quantify the accuracy of the EEG-PRISM attribution mapping. We note that the classification task, itself, is trivial for all five foundation models (Appendix \ref{section:model_acc}), which provides confidence that the models are learning discriminative representations from the data. 

For the spectral analysis, we compute the total positive attribution within frequency bands corresponding to each predicted class label. We then compare the band with the maximum attribution to the ground-truth spectral component used during signal generation. Since there is a direct association between frequency band and ground-truth class labels, all four explainability methods achieve near-perfect accuracy in Table~\ref{tab:sim_results} (left), with minor deviations observed for IG and DeepLift. DeepSHAP achieves the best performance and consistently recovers the ground-truth dominant frequency band from attribution scores. LRP performs comparably on EEGNet and ABNet, but we are unable to find a Python implementation for LRP that is stable for the remaining transformer-based foundation models. Overall, EEG-PRISM can accurately identify the spectral features that are associated with the underlying group differences. This performance is consistent \textit{across multiple models and explainers}.

For the spatial analysis, we use EEG-PRISM to map the input attributions onto 5,526 cortical and subcortical sources defined on the MNE fsaverage template. Although the EEG data was generated using 74,382 sources in the forward model, we reduced the dimensionality for the inverse problem for stability given the ill-posed nature of this problem. In contrast to the spectral evaluation, there no association between the source location and class label, as the former is randomly selected for each EEG input. To quantify the performance of EEG-PRISM, we first identify the spatial locations with the top 5\% of positive attribution scores for a given EEG recording. We then compute the center of mass for the retained locations and evaluate the the Euclidean distance between this center of mass and the true simulated source location. These distances are aggregated across the simulated dataset. Across methods, the average localization error is approximately $3 \pm 2$~cm, indicating generally accurate recovery.

To provide a more nuanced evaluation, we compute quadrant-level accuracy by first assigning the center of mass to one of five regions: left anterior, right anterior, left posterior, right posterior, and subcortical/central (chance level = 20\%). We then compute the lobe-level accuracy of this assignment relative to the ground truth. As shown in Table~\ref{tab:sim_results}, all methods perform above chance in both datasets. IG achieves the best overall performance, indicating that gradient-based methods generally outperform others in this setting with approximate inverse solutions. Among the models, Labram shows better spatial accuracy followed by Cbramod but are not statistically different from each other. Overall, EEG-PRISM correctly localizes the attributions to sources that drive EEG, even in a high-dimensional spatial domain and when those locations are not directly associated with task labels. This enables robust post-hoc EEG analysis.

\begin{figure*}[t]
    \centering
    \includegraphics[width=0.67\linewidth]{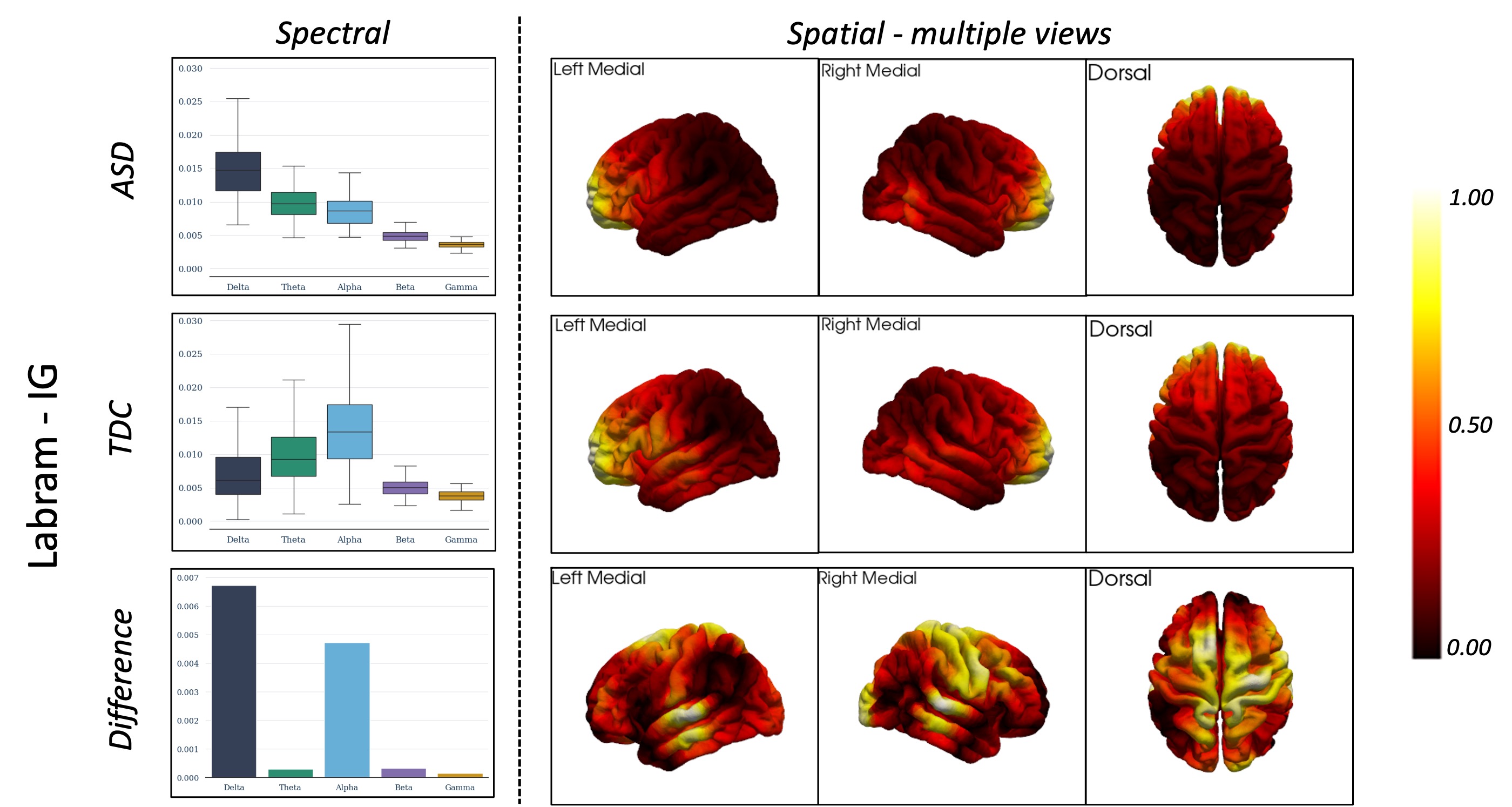}
    \caption{EEG-PRISM results on ACE dataset grouped into ASD subjects (top), TDC subjects (middle), and respective differences across two groups (bottom). Spectral attributions are shown as box plots (left), and spatial attributions are mapped onto the  3D surface (right).}
    \label{fig:ace}
\end{figure*}

\subsection{TUSZ Dataset: Window-Level seizure analysis}

We investigate whether EEG-PRISM can accurately reveal the clinically relevant spectral and spatial information associated with an epileptic seizure. As noted in Section~II.B, the foundation models are trained for window-level seizure versus non-seizure classification and do not directly use spectral or spatial information. We note that the foundation models achieve seizure classification accuracies near $88\%$ (Appendix~\ref{section:model_acc}), which is on par with the literature~\cite{zabihi2026transparent}. After training, we use IG to obtain the time-channel attribution scores on test sets. IG was selected due to its computational efficiency, applicability across all models, and consistent performance in the simulation study. 

In the spectral domain, we selected the correctly classified seizure windows in each EEG recording and averaged the positive EEG-PRISM spectral attributions across the five standard frequency bands. We do not include negative attributions, as they are inversely related to the predicted class. As shown in Fig.~\ref{fig:tuh} (top), delta and theta bands consistently exhibited higher importance, with theta being dominant in three of the five models. These findings align with clinical annotations: 43\% of subject reports abnormalities in the theta band and 35\% also mention the delta band. These bands are commonly associated with epileptiform activity such as spike-sharp waves following by slow waves~\cite{emmady2025eeg}. In contrast, the alpha and beta bands were referenced in 12\% and 11\% of clinical reports, respectively, and received lower attribution scores in that order. The gamma band showed minimal importance, likely due to the preprocessing steps that filtered out frequencies above 30 Hz. In summary, the spectral patterns uncovered by EEG-PRISM are consistent with clinical reports across foundation models and can be leveraged for early prediction, seizure tracking, and clinical stratification. 

In the spatial domain, we analyzed the first 15 windows (30 seconds) of seizure activity for each recording to determine whether EEG-PRISM localize seizure onset zone (SOZ). Here, we used EEG-PRISM to obtain attribution scores over 5,524 cortical and subcortical sources. Similar to the simulated experiment, we aggregated the positive attributions into five regions: left/right anterior, left/right posterior, and central. The region with the highest attribution was taken as the predicted SOZ and compared against clinical ground truth obtained from the report. As shown in Fig.~\ref{fig:tuh} (bottom), the resulting confusion matrix exhibits a diagonal-dominant structure, which indicates meaningful localization performance. We observe that central regions were the greatest source of confusion, likely due to the difficulty in estimating subcortical sources from EEG. Another source of error was confusion between the right and left hemispheres, which may be attributed to strong inter-hemispheric connections in the brain. Overall, EEG-PRISM achieved an average localization accuracy of roughly 50\% across all models, which is comparable to \textit{supervised methods} for SOZ localization~\cite{m2023deepsoz}. Notably, this performance is achieved without any prior knowledge of the SOZ when training the foundation models. Furthermore, EEG-PRISM enables SOZ localization in 3-D source space, which provides a finer spatial resolution for presurgical planning in epilepsy. 

\subsection{ACE Dataset: Group-Level Autism Analysis }

Our final experiment demonstrates how EEG-PRISM can be paired with foundation models for exploratory group-level biomarker analyses. We consider the problem of ASD versus TDC classification. Here, the foundation models perform a binary classification for each 2-second EEG window and then aggregate the predictions to the subject level via majority voting across all windows. The LaBraM foundation model achieves the highest classification accuracy (Appendix~\ref{section:model_acc}, Table~\ref{tab:model_acc}) and provides the backbone for EEG-PRISM. We choose best-performing models to ensure reliable interpretation of attribution scores; however, model performance does not dictate whether EEG-PRISM can be applied, as it operates post-hoc. Once again, we rely on IG to obtain the original time-channel attributions. 

For the spectral analysis, we aggregate the EEG-PRISM attributions into the five canonical frequency bands and average the positive attribution scores across EEG windows for a given subject. For the spatial analysis, we visualize 5,526 source-level attribution scores, averaged across subjects within each group. For visualization, we interpolate the scores using the NiLearn and PyVista packages~\cite{abraham2014machine,sullivan2019pyvista}. Similar to TUSZ, we restrict the analysis to correctly classified subjects and time windows for reliable interpretation.

As shown in Fig.~\ref{fig:ace} (left), correct ASD classification is primarily associated with increased attributions in the delta frequency band, whereas correct TDC classification is driven by alpha band activity. In the context of ASD, which is characterized by differences in social communication, increased delta activity is thought to reflect altered cortical processing, while disrupted alpha rhythms indicate differences in attention and cognitive engagement. Prior studies report changes in both delta power and alpha connectivity in ASD consistent with these patterns~\cite{wang2013resting,kopanska2025exploratory}. At the spatial level, both the ASD and TDC groups exhibit similar attribution patterns, with frontal lobe driving the classification. The group-level differences in spatial attributions tend to localize in the temporal lobes and extend into right central regions. Interestingly, this interpretation is  consistent with established literature: prefrontal cortex supports working memory and higher cognitive processing~\cite{elbaba2023frontal}, while the temporal lobes contribute to auditory processing, language, and social communication~\cite{patriquin2016neuroanatomical}. Thus, EEG-PRISM enables post-hoc biomarker discovery with foundation models consistent with prior work. 


\section{Discussion}

In this work, we addressed a key limitation of emerging foundation models for EEG analysis. While these models achieve strong performance using high-dimensional time–channel inputs, they are opaque with limited interpretability. Post-hoc ``explainers" partially mitigate this issue by identifying the input features associated with a given prediction. However, the input EEG (i.e., time-channel space) does not align with clinically meaningful spectral and spatial representations. Therefore, we have introduced EEG-PRISM as a principled framework to project the input-level attribution scores derived from \textit{mulitple post-hoc explainers} into other domains. EEG-PRISM uses a combination of signal processing theory with chain rule-based backpropagation to enable the mapping to be done \textit{without modifying or retraining the foundation model}. We have systematically evaluated the accuracy and fidelity of EEG-PRISM across five foundation models and four attribution methods. We have further validated its clinical utility in epilepsy by demonstrating its alignment with spectral and spatial patterns in clinical reports. Finally, we have demonstrated the potential of EEG-PRISM for group-level biomarker analysis using autism as a case study. Taken together, EEG-PRISM is a lightweight and ubiquitous tool that seamlessly works across foundation models, explainability methods, and clinical/research domains. 

Intuitively, mapping the time-channel attributions to the spectral domain is straightforward due to the invertible DFT incorporated within EEG-PRISM. In simulated experiments, this mapping capitalizes on the association between frequency band and class label to achieve near-perfect recovery of the ground truth phenomenon. In epilepsy, EEG-PRISM consistently identified delta–theta contributions to seizure detection, which aligns with the established clinical knowledge of epileptiform activity~\cite{emmady2025eeg,clemens2000eeg}. In our exploratory autism analysis, EEG-PRISM identified interactions between delta and alpha bands as influencing model predictions. Delta slowing is commonly associated with atypical cortical processing, while alpha rhythms are linked to wakeful rest, attention, and cognitive engagement. Prior studies in autism have reported alterations in both delta power and alpha connectivity in autism~\cite{wang2013resting,kopanska2025exploratory}. 
More broadly, spectral EEG features are closely tied to brain function: delta activity relates to sleep and pathological slowing, theta to memory and cognitive control, alpha to attention and inhibitory processing, beta to motor activity, and gamma to higher-order cognition and binding \cite{buzsaki2006rhythms}. 
EEG-PRISM enables any foundation model to be interpreted in this space, ensuring that model explanations remain aligned with established neuroscientific principles and clinical patterns.

The mapping from input to spatial source domain using EEG-PRISM relies on user-defined forward and inverse models. Despite this challenge, it achieved high accuracy in simulated experiments with known ground truth phenomena, thus demonstrating robust \textit{post-hoc recovery} of spatial patterns. EEG-PRISM also achieved competitive SOZ localization performance in a fully post hoc setting, despite the foundation models not having access to spatial supervision during training. In fact, it performed comparably to prior supervised approaches for SOZ localization~\cite{m2023deepsoz}. These results highlight the potential of EEG-PRISM for clinical translation, including applications in treatment planning for epilepsy. Extending this framework to autism, we identified neurologically plausible associations in the  frontal and temporal areas. The prefrontal cortex is widely implicated in cognitive control and working memory, while temporal and central regions are associated with memory, auditory processing, language, and emotional regulation which likely  differ between individuals with ASD and TDCs~\cite{pelphrey2008brain,courchesne2007mapping}. 
Thus, EEG-PRISM can identify spatially meaningful regions that drive model predictions, thus providing an important foundation for noninvasive biomarker discovery and hypothesis generation in neuropsychiatric disorders.

Another way to evaluate the ``information gain" of EEG-PRISM is via the mean-normalized range of the attribution scores in each domain. Specifically, in cases where feature strongly influences the output, the attribution scores will be uniformly distributed, leading to low range. Conversely, if there are a few dominant features, then the attribution scores will be more concentrated and increase the range. Mean normalization ensures that the values are comparable across vectors of different sizes. Table~\ref{tab:cv} reports the mean-normalized ranges for the Labram model and IG explainer in the original time-channel domain, the spectral domain, and the spatial source domain. The default time-channel domain exhibits the lowest mean-normalized range, indicating that it is \textit{not} the natural space to learn predictive representations. Across datasets, the spectral domain shows substantially higher mean-normalized range. The spatial domain also shows higher values than input space, particularly in ACE dataset where we saw that the information was localized to frontal regions. Overall, our analysis suggests that there is more concentrated information in the spectral and spatial domains, as often used in clinical practice.

\bgroup
\renewcommand{\arraystretch}{1.5}
\begin{table}[t]
\centering
\caption{mean-normalized range of EEG-PRISM attributions}
\begin{tabular}{|c|c|c|c|}
\hline
Dataset  & Simulated  & TUSZ & ACE  
         \\ \hline
{Input} &8.137 & 4.163 &  3.239 \\ \hline

{Spectral}& \textbf{31.691}   & \textbf{215.192} &  6.897 \\ \hline
{Spatial}& 13.395   & 7.714  & \textbf{8.981}   \\ \hline
\end{tabular}
\label{tab:cv}
\end{table}
\egroup

In this work, we applied EEG-PRISM to a variety of foundation models and explainers to demonstrate its universal applicability. Currently, larger pretrained models, such as Labram, Cbramod, and LUNA, provide a strong starting point and adapt quickly with small datasets like ACE. For real-world datasets with well-understood tasks (e.g., seizure detection), model choice matters less, i.e., most models perform similarly. 
Across models, attribution scores are broadly consistent, though we expect that the best model/explainer configuration will differ across real-world datasets. Our work demonstrates that regardless of the foundation model and downstream task, we can use EEG-PRISM to obtain more clinically relevant interpretations.

Likewise, we have shown that EEG-PRISM pairs well with a range of backpropagation-based attribution methods. Notably, IG performs well on models with smooth gradients, has strong axiomatic properties, and produces consistent results.
DeepLIFT addresses gradient saturation using discrete differences and model-specific rules. However, it can show instability with a single baseline. DeepSHAP approximates Shapley values by combining DeepLIFT with multiple baselines. It performed best for deterministic spectral attributions. However, it showed slight degradation in the spatial domain, likely due to error magnification across multiple baselines. As foundation models increasingly adopt transformer architectures, EEG-PRISM with IG presents as a stable, model-agnostic attribution approach across datasets. LRP was not applied to transformers in this study, as LRP rules for self-attention remain an active area of research. 

One limitation of this work is that we analyzed spectral and source domains separately. Future work should explore joint spectral-source representations as a methodological extension while ensuring that approximate inversion does not compromise spectral interpretation. 
Another limitation is that our theoretical guarantees were derived for linear mappings only. Future work will explore mappings for complex nonlinear transformations, thus allowing researchers to investigate EEG connectivity and their relationship with model outputs. Approaches such as linear surrogate models and autoencoder-based inversion may provide practical strategies for approximating these transformations. Finally, using a simple average to summarize attribution patterns across EEG windows can obscure clinically-relevant signal variability. Thus, we will explore algorithmic improvements to identify consistent attributions across windows and across subjects for group-level analyses. Future work will apply EEG-PRISM to other neurological conditions, including disease subtyping and treatment response monitoring by tracking identified multi-domain biomarkers.



\section{Conclusion}

In summary, we presented EEG-PRISM as a principled and flexible framework for interpreting EEG foundation models by projecting attribution scores into clinically meaningful spectral and spatial domains. Across simulated and real-world datasets, EEG-PRISM demonstrated strong fidelity in recovering relevant frequency patterns and spatial regions, aligning with established neuroscientific and clinical knowledge. Despite operating in a fully post hoc and model-agnostic setting, it achieved competitive performance in tasks such as seizure characterization and localization, while also showing promise in exploratory analyses such as autism. By bridging the gap between high-performing black-box models and interpretable neurological insights, EEG-PRISM provides a practical pathway toward more transparent, reliable, and clinically grounded EEG analysis, with broad potential for future research and translational applications.

\appendices

\section*{Acknowledgment}
AI tools assisted with general code completion and manuscript formatting. All content were thoroughly reviewed by the authors. We thank the ACE Network and GENDAAR for collecting EEG data.

\section{Foundation Model Performance on the Primary Prediction Task} \label{section:model_acc}

Table~\ref{tab:model_acc} reports the average classification accuracy of the primary task of five foundation models in simulated, TUSZ, and ACE datasets. In simulated and TUSZ datasets, window-level accuracy of abnormality  (4-way) and seizure classification (binary)  are reported, respectively, where all models achieved similar performance.  In ACE dataset, subject level accuracy in ASD vs TDC binary classification is reported after majority voting of window-level model output across subject's EEG recording. Labram achieved the the highest accuracy (66\%), followed  by LUNA (62\%).

\bgroup
\renewcommand{\arraystretch}{1.5}
\begin{table}[!ht]
\centering
\caption{model accuracy in classification across datasets}
\begin{tabular}{|c|c|c|c|}
\hline
Model & Simulated  & TUSZ & ACE 
         \\ \hline
{EEGNet}& 0.995$\pm$0.006 & 0.878$\pm$0.033  & 0.532$\pm$ 0.023    \\ \hline

{ABNet}& 1.000$\pm$0.000&  0.878$\pm$0.026 &  0.596$\pm$ 0.036    \\ \hline
{Labram}&1.000$\pm$0.000& 0.864$\pm$0.039    & \textbf{0.660 $\pm$ 0.072}    \\ \hline
Cbramod & 1.000$\pm$0.000  & 0.868$\pm$0.043   & 0.582$\pm$ 0.062 \\ \hline
LUNA & 1.000$\pm$0.000  & 0.878$\pm$0.027   & 0.624$\pm$ 0.043 \\ \hline
\end{tabular}
\label{tab:model_acc}
\end{table}
\egroup




\end{document}